\documentclass[letterpaper, 10 pt, conference]{ieeeconf} 

\IEEEoverridecommandlockouts                              

\usepackage[T1]{fontenc}    
\usepackage[backend=bibtex,style=ieee,mincitenames=1,maxcitenames=2,maxbibnames=20]{biblatex}
\usepackage{amsmath,amssymb,amsfonts}
\usepackage{algorithmic}
\usepackage{graphicx} 
\graphicspath{{figs/}} 
\usepackage{textcomp}
\usepackage[keeplastbox]{flushend}
\usepackage{subcaption} 
\usepackage{color} 
\usepackage{amsmath} 
\usepackage{amssymb}  
\usepackage{tikz-cd}
\usepackage{multirow}
\usepackage{booktabs,etoolbox} 
\usepackage{stfloats} 
\usepackage{footnote}
\usepackage{xcolor}
\usepackage{bm}
\def\BibTeX{{\rm B\kern-.05em{\sc i\kern-.025em b}\kern-.08em
    T\kern-.1667em\lower.7ex\hbox{E}\kern-.125emX}}

\newtheorem{theorem}{Theorem}
\newtheorem{lemma}{Lemma}

\begin{document}

\title{\LARGE \bf 3D Radar Velocity Maps for Uncertain Dynamic Environments
}

\author{Ransalu Senanayake$^{*1}$, Kyle Beltran Hatch$^{*1}$, Jason Zheng$^{2}$, and Mykel J. Kochenderfer$^{1}$
\thanks{$^{*}$Equal contribution.}
\thanks{$^{1}$R. Senanayake, K. B. Hatch, and M. J. Kochenderfer are with the Stanford Intelligent Systems Laboratory (SISL) in the Aeronautics and Astronautics Department, Stanford University, 496 Lomita Mall, Stanford, CA 94305, USA. Email: 
        {\tt\small \{khatch, ransalu, mykel\}@stanford.edu}.}
\thanks{$^{2}$J. Zheng is with the Department of Computer Science, Stanford University. Email: {\tt\small jzzheng@stanford.edu}.}
}

\maketitle

\begin{abstract}
Future urban transportation concepts include a mixture of ground and air vehicles with varying degrees of autonomy in a congested environment. In such dynamic environments, occupancy maps alone are not sufficient for safe path planning. Safe and efficient transportation requires reasoning about the 3D flow of traffic and properly modeling uncertainty. Several different approaches can be taken for developing 3D velocity maps. This paper explores a Bayesian approach that captures our uncertainty in the map given training data. The approach involves projecting spatial coordinates into a high-dimensional feature space and then applying Bayesian linear regression to make predictions and quantify uncertainty in our estimates. On a collection of air and ground datasets, we demonstrate that this approach is effective and more scalable than several alternative approaches.
\end{abstract}

\section{Introduction}

\begin{figure*}[]
     \begin{subfigure}[b]{0.3\linewidth}
         \centering
         \includegraphics[width=1.\linewidth]{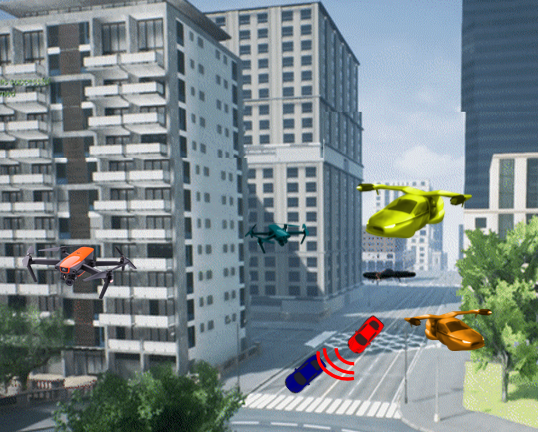}
         \caption{A future urban environment.}\label{fig:1a}	
    \end{subfigure}
    \begin{subfigure}[b]{0.34\linewidth}
         \centering
         \includegraphics[width=1.\linewidth]{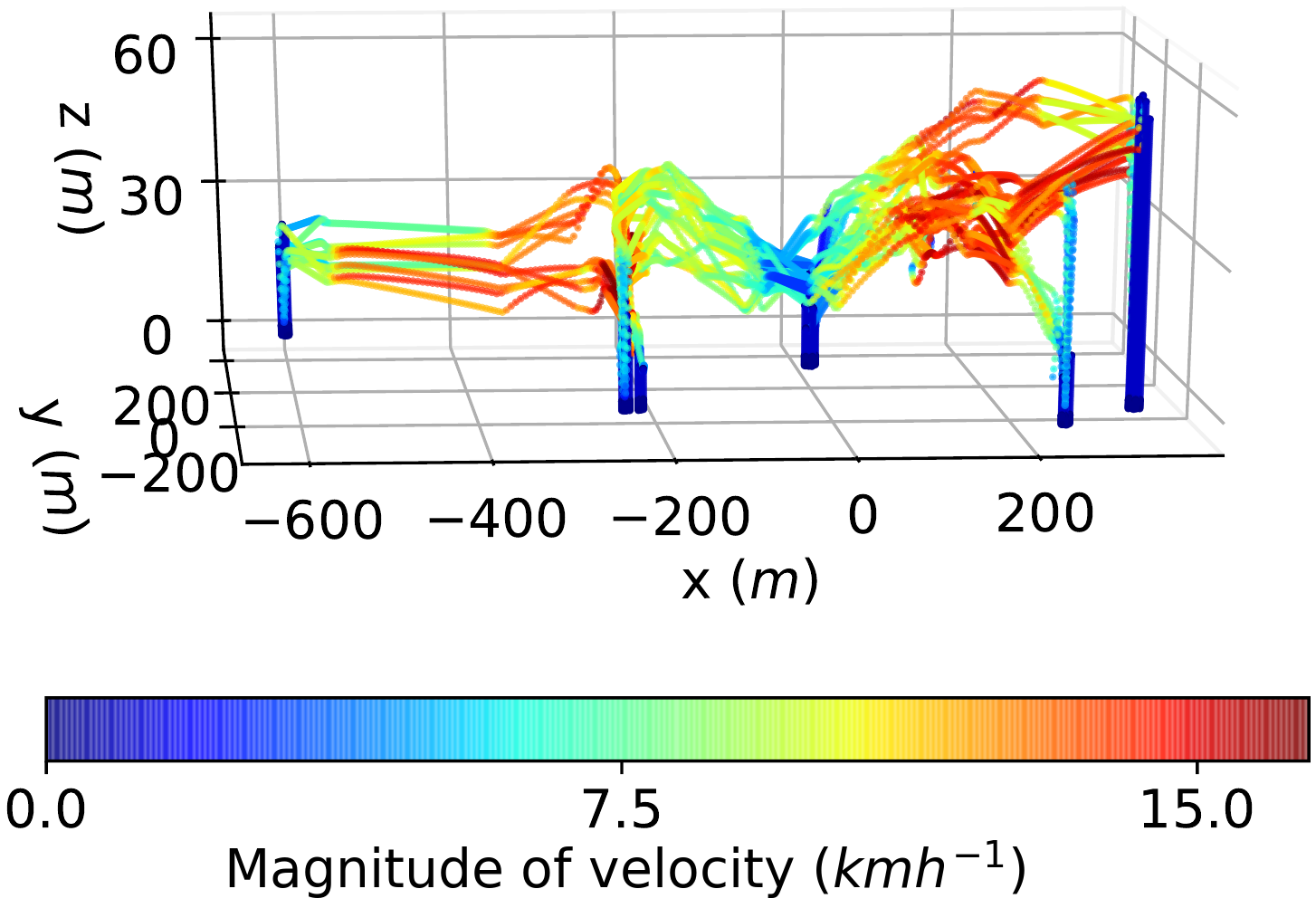}
         \caption{Velocity of drone trajectories.}\label{fig:1b}
    \end{subfigure}
    \begin{subfigure}[b]{0.34\linewidth}
         \centering
         \includegraphics[width=1.\linewidth]{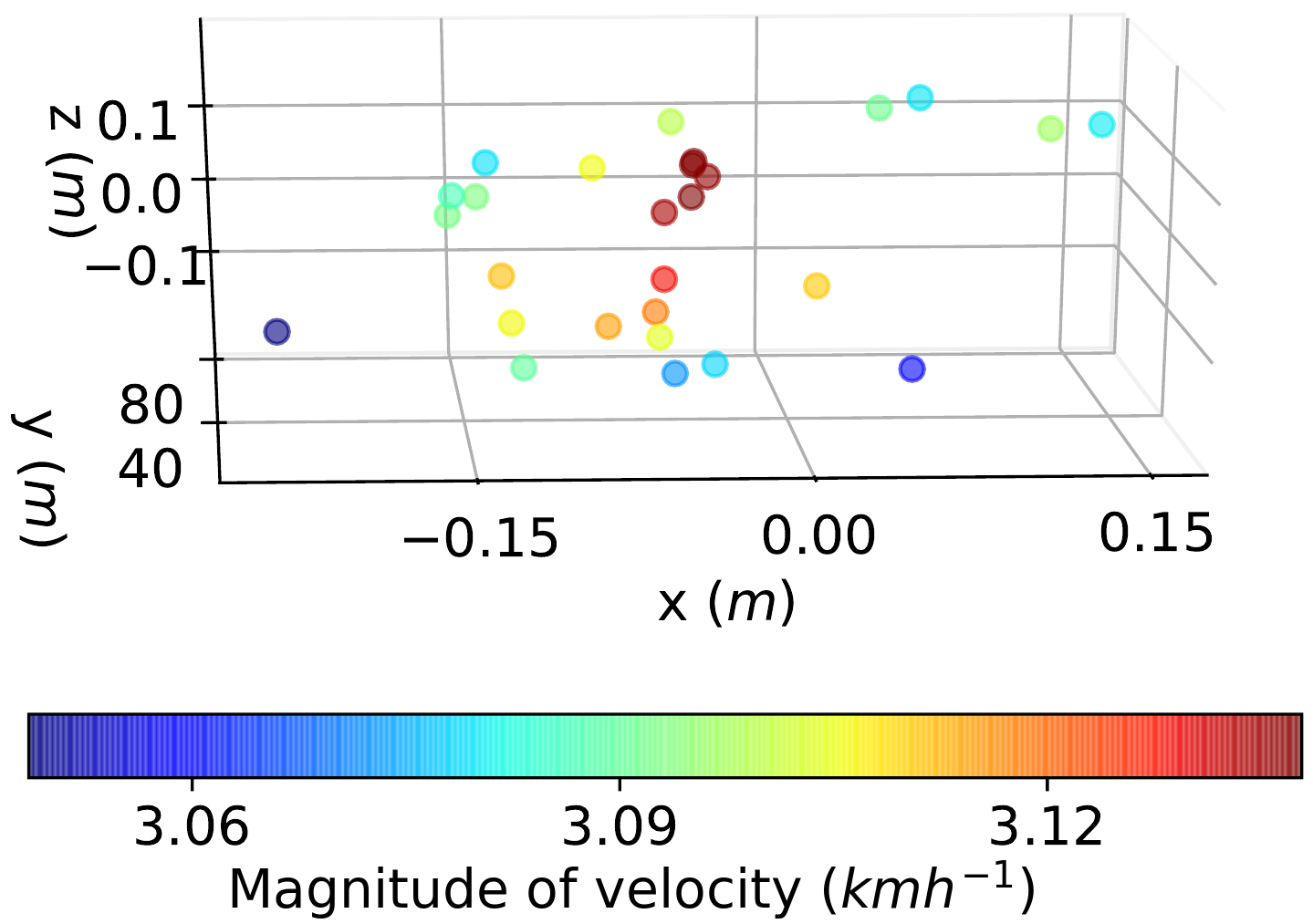}
         \caption{Radar point cloud velocity of the red car.}\label{fig:1c}	
    \end{subfigure}
 \caption{(a) Since we will have complex urban environments with ground vehicles, delivery drones, and urban air mobility (UAM) in the future, it is important to model the dynamics such as velocity and acceleration of the environment. (b) We need to build \textit{macroscopic} models of both ground and air roads for the whole city from \textit{large} amounts of trajectory data. (c) Vehicles also need to build \textit{microscopic} models of the velocity of objects around them from \textit{small} amounts of sparse data. Being able to learn the velocity \textit{quickly} helps making efficient decisions. Being able to quantify the associated \textit{uncertainty} aids in making safe decisions.}
 \label{fig:fig1}
\end{figure*}

Future urban transportation system concepts include a mixture of both autonomous and human-driven vehicles. We anticipate driverless cars on roads as well as  delivery drones~\cite{choudhury2020efficient} and urban air mobility (UAM) systems with vertical take-off and landing capabilities~\cite{holden2016}. These advances add complexity to our transportation systems (Figure~\ref{fig:fig1}a), making decision making for autonomous vehicles operating in such dynamic systems even more challenging. 

Safe and efficient transportation in a congested environment requires accurately modeling the flow of traffic in 3D space at both the macroscopic and microscopic levels. At the macroscopic level, the velocity maps estimate the average behavior of vehicles at different locations in the system, as opposed to estimating the current behavior of surrounding targets (Figure~\ref{fig:fig1}b). We may have large amounts of data to build such models, and we want to be able to build them efficiently. At the microscopic level, we want to model the surroundings of an individual vehicle from a continuous stream of data to allow it to maneuver safely (Figure~\ref{fig:fig1}c). These velocity maps estimate the behavior of vehicles surrounding ego vehicle at the current point in time. We often only have to learn models from very little data.

To provide robust control of these vehicles, it is important to capture the uncertainty associated with our velocity maps. For every point in the 3D space, for example, it would be helpful to quantify both the mean and variance of the various velocity components. While Bayesian nonparametric methods such as Gaussian process-based models~\cite{rasmussen2003gaussian} can provide uncertainty estimates, they (as well as many of their approximations~\cite{senanayake2017learning}) are generally not suitable for building large-scale macroscopic models because their computational complexity grows too quickly with the number of data points~\cite{rasmussen2003gaussian,o2012gaussian}.

This paper applies an approach that has many of the desirable attributes of a Gaussian process-based model, but it is faster and more memory efficient. We adopt an approach that involves projecting the spatial coordinates into a high-dimensional, yet interpretable, feature space to capture information local to a given area. Bayesian linear regression is used to learn the parameters of the model. We can then query arbitrary points in the 3D space to obtain smooth estimates of the mean and variance of the velocity components. The model can be efficiently updated online, making it amenable to changes in the environment. A theoretical analysis shows that this model is robust against input noise. We demonstrate our approach on simulated and real-world datasets.

\section{Related Work}

Many robotics applications involve building maps of the environment. Occupancy maps is the most common representation. Occupancy maps alone can only be used to navigate through an environment when surrounding obstacles are stationary. However, in real urban environments an autonomous vehicle must be able to safely navigate around both stationary obstacles and moving vehicles. Developing velocity maps are therefore crucial for planning algorithms to plan safe trajectories for autonomous vehicles operating in urban environments.

Techniques such as occupancy grid maps~\cite{elfes1989using} discretize the environment and model if a cell is free or occupied. Such models assume the cells are independent, ignoring neighborhood information. While such methods can model the binary occupancy probability, they do not model epistemic uncertainty. To address these limitations, several 2D Bayesian models have been proposed~\cite{o2012gaussian, senanayake2017bayesian,mcleod2019navigating,duong2020autonomous}. These models can be queried at an arbitrary resolution at run time. There have also been various attempts to model occupancy in dynamic environments~\cite{senanayake2016spatio,itkina2019dynamic,Toyungyernsub2020arxiv}. However, such models do not explicitly represent the velocity as a map. We model the uncertainty of velocity in the 3D space which is computationally challenging compared to conventional 2D occupancy mapping techniques. Although there are similarities with the model we explore in this paper, these models use a binary random variable to model occupancy~\cite{senanayake2018building}. In this work, we are interested in modeling the velocity which is not a binary variable. 

Velocity modeling has previously been studied in various disciplines~\cite{lawrance2011path,homicz2002three}. However, these models are deterministic. Other models that attempt to estimate quantities that change temporally include modeling the long-term occupancy~\cite{senanayake2017learning} and directions~\cite{Senanayake2020itsc} in 2D. In contrast, the objective of this paper is modeling velocity with their associated epistemic uncertainties in 3D space.

\section{Bayesian Dynamic Fields}

This section introduces the proposed framework for modeling 3D velocity maps. Since we want to model the spatial field of dynamics such as velocity with its associated uncertainty, we refer to this framework as Bayesian Dynamic Fields (BDF). First, we explain how to generate high dimensional features from data using kernel functions. Then, we discuss how to build a linear model from these features and estimate uncertainty with Bayesian linear regression with these basis functions. These basis functions can be customized for different datasets. This model is applied to build macroscopic and microscopic dynamic models. Finally, a theoretical analysis of the robustness of the proposed framework is presented.

\subsection{High dimensional feature space}
\label{sec:kernels}

Our objective is to model the velocity field and its corresponding uncertainty field in a given 3D space. Velocity variations exhibit nonlinear patterns with respect to spatial location. A common tool for modeling nonlinear patterns is deep neural networks, but they generally require large amounts of training data and can be slow to train. We focus on kernel methods~\cite{kivinen2002online}, which have been successfully used to model spatial quantities such as occupancy~\cite{senanayake2017bayesian,vallicrosa2018h,ramos2016hilbert} and directions~\cite{zhi2019spatiotemporal} in robotics, soil concentration in geostatistics~\cite{cressie2015statistics}, and disease propagation in epidemiology~\cite{senanayake2016predicting}.

Kernels are similarity functions. A kernel, $k(\mathbf{x}_a,\mathbf{x}_b)$, takes two inputs $\mathbf{x}_a$ and $\mathbf{x}_b$ and outputs a measure on how similar the two inputs are. In this work, we use the squared-exponential kernel because of its simplicity and interpretability:
\begin{equation}
    k(\mathbf{x}_a,\mathbf{x}_b) = \exp(-\gamma \| \mathbf{x}_a - \mathbf{x}_b \|_2^2),
    \label{eq:kernel}
\end{equation}
where $\gamma$ is the inverse bandwidth hyperparameter. This hyperparameter controls the sensitivity of the similarity. Kernels with smaller values of $\gamma$ capture correlations over larger areas. 

We use this kernel to define a set of $M$ basis functions, $k(\mathbf{x}, \tilde{\mathbf{x}}_1), \ldots, k(\mathbf{x}, \tilde{\mathbf{x}}_M)$, where $\mathbf{x}$ is any point in 3D space and $\tilde{\mathbf{x}}_1, \ldots, \tilde{\mathbf{x}}_M$ are fixed points in that space. We arrange the fixed points in a 3D regular grid. Although kernels can alternatively be computed through random Fourier features~\cite{mildenhall2020nerf,ramos2016hilbert} or Nystrom approximation, we use a grid for simplicity, interpretability, and accuracy~\cite{ramos2016hilbert}. If desired, these fixed points can be learned alongside other parameters~\cite{senanayake2018automorphing}.

A data point $\mathbf{x} \in \mathbb{R}^3$ may come, for example, from radar (Figure~\ref{fig:fig2}a--b) or IMU measurements. For $N$ such data points, $\mathbf{X} = \{ \mathbf{x}_n \}_{n=1}^N$, the feature matrix $\Phi(\mathbf{X}) \in \mathbb{R}^{N \times M}$ is defined as,
\begin{equation}
    \Phi(\mathbf{X}) = 
    \begin{bmatrix}
       k(\mathbf{x}_1, \tilde{\mathbf{x}}_1) & k(\mathbf{x}_1, \tilde{\mathbf{x}}_2) & \dots & k(\mathbf{x}_1, \tilde{\mathbf{x}}_M)\\
       k(\mathbf{x}_2, \tilde{\mathbf{x}}_1) & k(\mathbf{x}_2, \tilde{\mathbf{x}}_2) & \dots & k(\mathbf{x}_2, \tilde{\mathbf{x}}_M)\\
       \vdots & \vdots & \ddots & \vdots  \\
      k(\mathbf{x}_N, \tilde{\mathbf{x}}_1) & k(\mathbf{x}_N, \tilde{\mathbf{x}}_2) & \dots & k(\mathbf{x}_N, \tilde{\mathbf{x}}_M)
    \end{bmatrix}.
    \label{eq:features}
\end{equation}
This matrix would be mostly sparse with many values close to zero because the kernel function goes to zero when its two inputs are far apart.

\begin{figure}[]
     \begin{subfigure}[b]{1.00\linewidth}
         \centering
         \includegraphics[width=1.\linewidth]{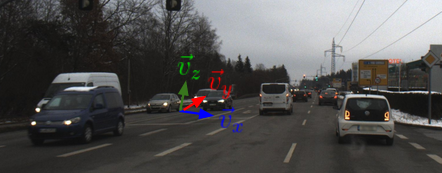}
         \caption{Field of view.}\label{fig:2a}	
    \end{subfigure}
    \begin{subfigure}[b]{0.58\linewidth}
         \centering
         \includegraphics[width=1.\linewidth]{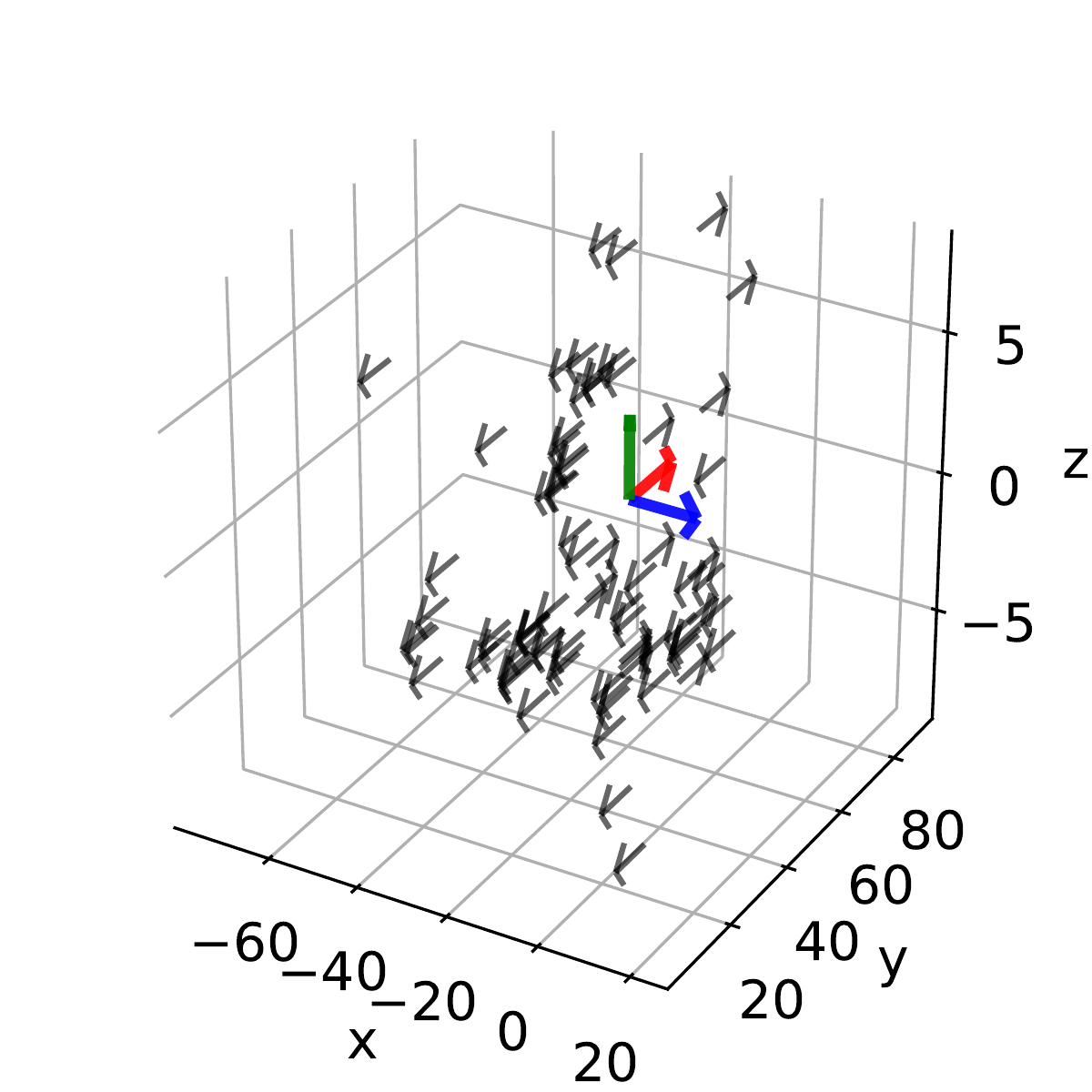}
         \caption{Radar measurements.}\label{fig:2b}
    \end{subfigure}
    \begin{subfigure}[b]{0.4\linewidth}
         \centering
         \includegraphics[width=1.\linewidth]{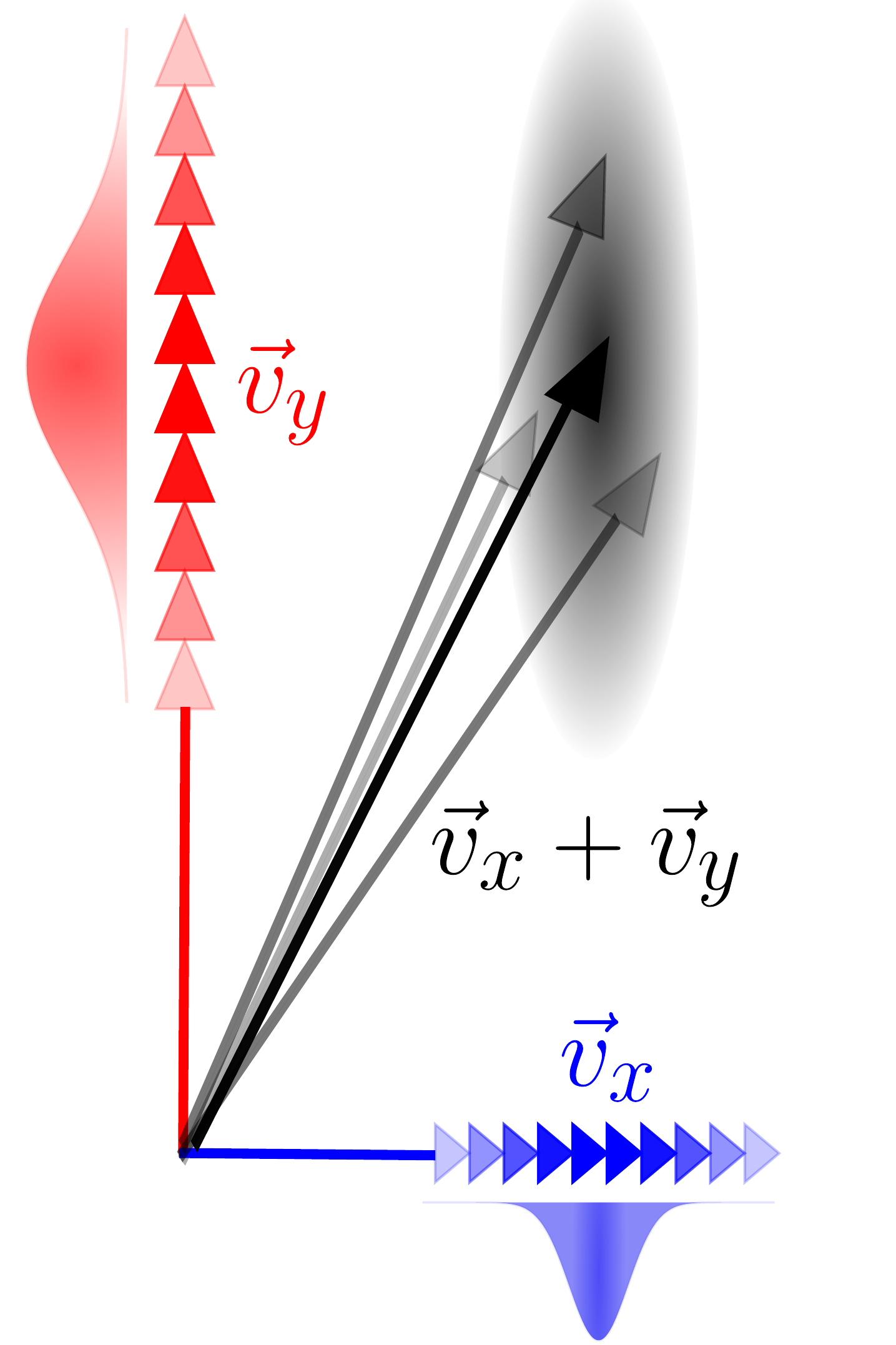}
         \caption{Uncertainty of velocity.}\label{fig:2c}
    \end{subfigure}
 \caption{Unlike LIDAR, automotive radars provides the velocity associated with each point in the point cloud. We want to estimate the uncertainty of velocity at the arbitrary point indicated by the three colored arrows. (a) The camera image of an area an automotive radar can see. (b) The velocity vectors from pre-processed 3D radar measurements~\cite{meyer2019automotive} are in black. (c) The estimates from our model are not deterministic vectors but distributions of velocities for each direction. For simplicity, we only show the red and blue directions. A few resulting vectors are shown in black.}
 \label{fig:fig2}
\end{figure}

\subsection{Bayesian inference}
\label{sec:bayes}

We want to build a model that can estimate the velocity of a given point in the environment. For instance, as shown in Figure~\ref{fig:fig2}, given some sparse velocity measurements, we want to know the velocity at an arbitrary location indicated by the three colored arrows. Because the three directional velocity components ${v}^{(x)}$, ${v}^{(y)}$, and ${v}^{(z)}$ are independent with each other, three different models are learned in parallel.

If the velocity component labels of each datapoint $\mathbf{x}$ is denoted by $v$, then the training dataset can be defined as $\mathcal{D}=\{ (\mathbf{x}_n,v_n) \}_{n=1}^N = (\mathbf{X}, \mathbf{v})$. Since we projected the data into $M$-dimensional space, we can now create a linear model $v = \mathbf{w}^\top \Phi(\mathbf{x}) + \epsilon$ with noise $\epsilon \sim \mathcal{N}(0, \beta^{-1})$ where $\beta$ is the noise precision. Our objective is to learn the parameter vector $\mathbf{w} \in \mathbb{R}^M$ from $\mathcal{D}$. The velocities at one of the three directions in a given location are modeled as a Gaussian distribution. Because measurements are \emph{i.i.d.}, the likelihood can be decomposed as $p(\mathbf{v} \vert \mathbf{w},\mathbf{X},\beta) = \prod_{n=1}^N \mathcal{N}(v_n \vert \mathbf{w}^\top \Phi(\mathbf{x}_n), \beta^{-1})$.

As illustrated in Figure~\ref{fig:fig2}c, we are not only interested in estimating the velocity but also the associated uncertainty. In order to model the epistemic uncertainty, we consider a prior probability distribution over $\mathbf{w}$. A Gaussian distribution over $\mathbf{w} \sim \mathcal{N}(\bm{\mu}_0, \bm{\Sigma}_0)$ is a conjugate prior to the likelihood model of our interest. With this prior, the posterior distribution $p(\mathbf{w} \vert \mathbf{X},\mathbf{v}) = \mathcal{N}(\mathbf{w} \vert \bm{\mu}_1,{\bm{\Sigma}}_1)$ from Bayesian linear regression in the feature space can be computed analytically~\cite{bishop2006pattern}:
\begin{align}
    \bm{\mu}_1 &= \bm{\Sigma}_1(\bm{\Sigma}_0^{-1} \bm{\mu}_0 + \beta {\Phi^\top(\mathbf{X})} \mathbf{y}) 
\label{eq:post_mu}\\
    \bm{\Sigma}_1 &= (\bm{\Sigma}_0^{-1} + \beta {\Phi^\top(\mathbf{X})} {\Phi(\mathbf{X})})^{-1}.
    \label{eq:post_sig}
\end{align}

Given that our prior knowledge about an area is minimal, we can set $\bm{\mu}_0=\mathbf{0}$ and $\Sigma_0 = \alpha^{-1} \mathbf{I}$, where $\alpha$ is a small parameter indicating the precision (i.e. the inverse of variance) of the prior. This indicates an almost uninformative prior. Furthermore, this prior acts as a natural regularizer for the high-dimensional regression problem. 

Having obtained the posterior distribution, the posterior predictive distribution $p(v_* \vert \mathbf{x}_*, \mathcal{D}, \alpha, \beta) = \mathcal{N}(v_* \vert \bm{\mu}_*,\bm{\sigma}_*)$ for an arbitrary unknown point $\mathbf{x}_* \in \mathbb{R}^3$ can be queried analytically:
\begin{align}
    \bm{\mu}_* &= \bm{\mu}_1^\top \Phi(\mathbf{x_*})\\
    \bm{\sigma}_*^2 &= \beta^{-1} + {\Phi}^\top(\mathbf{x_*}) \bm{\Sigma}_{1} {\Phi(\mathbf{x_*})}.
\label{eq:pred_sig}
\end{align}
If we have a batch of query points, $\Phi$ can be computed as in (\ref{eq:features}).

\subsection{Dimension-adjusted kernels}
\label{sec:ard}
When defining the kernel in (\ref{eq:kernel}), we considered that the norm between a data point $\mathbf{x} \in \mathbb{R}^3$ and a fixed point $\tilde{\mathbf{x}} \in \mathbb{R}^3$ is scaled by the hyperparameter $\gamma$. However, if the nonlinear patterns change in different rates in different axes of the 3D space, we can scale them differently to better fit the model:
\begin{equation}
    k(\mathbf{x},\tilde{\mathbf{x}}) = \exp \big( - (\mathbf{x} - \tilde{\mathbf{x}})^\top \Gamma^{-1}  (\mathbf{x} - \tilde{\mathbf{x}}) \big),
\end{equation}
where
\begin{equation}
    \Gamma = 
    \begin{bmatrix}
    \gamma_x & 0 & 0\\
     0 & \gamma_y & 0\\
     0 & 0 & \gamma_z
    \end{bmatrix},
    \label{eq:ard}
\end{equation}
is the hyperparameter matrix. Although it is possible to consider the full non-diagonal matrix by considering the hyperparameter-hyperparameter covariances, in this application, we ignore such complex interactions for the sake of simplicity. It is also possible to learn the parameters~\cite{senanayake2018automorphing} or learn a completely new kernel function (\ref{eq:kernel})~\cite{gonen2011multiple}. 

\subsection{Macroscopic and microscopic velocity maps}
\label{sec:macromicro}
As shown in Figure~\ref{fig:fig1}b, a macroscopic model represents the velocity of a large area of an urban environment. These models will especially be useful when designing urban air mobility systems~\cite{thipphavong2018urban}. In order to build a global model, trajectory information of vehicles in the environment are collected for a long time period. This information can be obtained from surveillance radar~\cite{Jung2019} or IMU data. The velocity is represented as $\mathbf{v}=(v_x, v_y, v_z)$ and we perform inference separately for each of the different dimensions. For each of these learning problems, we place a regular grid over the entire space for fixed points. 

In macroscopic mapping, we have to learn large environments with lots of data. In such environments, when data arrives sequentially, it is possible to use the posterior distribution from the previous time step as the prior distribution in the current time step before applying the update rules in (\ref{eq:post_mu})--(\ref{eq:post_sig}). Furthermore, when data is obtained sequentially, we can start with the assumption that the environment is static and populate the 3D environment (training dataset) with quasi-Monte Carlo (QMC) samples of velocity zero to improve the learning efficiency. QMC sampling techniques are known to be sample efficient over Monte Carlo techniques~\cite{lemieux2009monte}. In particular, Sobol and generalized Halton QMC sequences can populate the 3D free space more evenly~\cite{lemieux2009monte,tompkins2019black}. Trajectory data points that are in the neighborhood measured by the Euclidean distance of the populated data are removed from the dataset. 

For the microscopic model, we are interested only in modeling the field of view of the ego vehicle (Figures~\ref{fig:fig1}c and~\ref{fig:fig2}). For such models, velocity information coming from automotive radar is used. Since automotive radar point clouds, unlike LIDAR measurements, are extremely sparse, modeling the epistemic uncertainty is important so that we know which of our velocity estimates are less reliable. Furthermore, automotive radar measurements tend to have higher noise levels and therefore, modeling the aleatoric uncertainty is equally important. Equation (\ref{eq:pred_sig}) is the combined aleatoric-epistemic uncertainty estimation.

\subsection{Wasserstein robustness against input noise}
\label{sec:robustness}

Sensor measurements, especially radar measurements, are typically corrupted by some noise. In this section, we theoretically analyze whether the proposed model can withstand input perturbations~\cite{kuhn2019wasserstein}.

\begin{lemma}
The squared 2-Wasserstein distance between a normal distribution $\mathcal{N}(\mathbf{x},\nu^2I)$ and a point $\mathbf{x}_m$ is $\mathcal{W}^2_2 = \| \mathbf{x} - \mathbf{x}_m \|_2^2 + \nu^2$. 
\end{lemma}
\begin{proof}
By computing $\inf \mathbb{E}[\| \mathrm{x} - \mathrm{x}_m \|^2_2]$ between $\mathcal{N}(\mathbf{x},C^2)$ and $\mathcal{N}(\mathbf{x}_m,C_m^2)$, it can be shown that $\mathcal{W}_2 = \| \mathbf{x} - \mathbf{x}_m \|_2^2 + \mathrm{Tr}\big( C^2 + C_m^2 -2(C_m C^2 C_m)^\frac{1}{2} \big)$~\cite{kuhn2019wasserstein}. For diagonal matrices $C_m$ and $C = \nu I$, by reducing the Gaussian to a Dirac delta distribution (by taking the limit of the variance terms to zero), we obtain $\mathcal{W}_2 = \| \mathbf{x} - \mathbf{x}_m \|_2^2 + \nu^2$.\hfill 
\end{proof}

\begin{theorem}
Expectation of estimations of the linear model $\mathbf{w}^\top \Phi(\mathbf{X})$ with weights $\{ w_m \}_{m=1}^M$ where $w_m \sim \mathcal{N}(\mu_m, \sigma_m)$ in Bayesian dynamic fields are unaltered by the input noise $\mathcal{N}(0,\nu^2)$ under the 2-Wasserstein metric.
\end{theorem}
\begin{proof} Let us rewrite the $\mathbf{w}^\top \Phi(\mathbf{X})$ introduced in Section~\ref{sec:bayes} as a summation,
\begin{align*}
\mathbf{y} &\approx \sum_{m=1}^M w_m k(\mathbf{x}, \mathbf{x}_m) \\
&= \sum_{m=1}^M w_m \exp\big( -\gamma (\|\mathbf{x} - \mathbf{x}_m \|_2^2 + \nu^2) \big) \; \text{(Lemma 1)} \\
&= \sum_{m=1}^M w_m \exp(-\gamma  \nu^2) \exp( -\gamma \|\mathbf{x} - \mathbf{x}_m \|_2^2) \\
&=\sum_{m=1}^M w^\prime_m \exp( -\gamma \|\mathbf{x} - \mathbf{x}_m \|_2^2) 
\end{align*}
The measurement noise is absorbed into the distributions $\{ w^\prime_m \}_{m=1}^M$ with $w^\prime_m \sim \mathcal{N}(\mu^\prime_m, \sigma^\prime_m)$ whose parameters are estimated during training. Therefore, the mean estimations are unaffected by noise. \hfill 
\end{proof}

\begin{table}[]
  \centering
  \caption{Datasets}
    \begin{tabular}{@{}lll@{}}
    \toprule
    Dataset & Source & Description \\
    \midrule
    Chunks & Synthetic & 3 closely packed velocity clusters\\
    Blobs &  Synthetic & 3 separated velocity clusters as blobs\\
    Carla &  Simulated & Automotive radar data \\
    Astyx & Real & Automotive radar data\\
    nuScenes & Real &  Automotive radar data\\
    AirSim &  Simulated & 60 drone trajectories to simulate UAM\\
    Airport & Real & 100 aircraft trajectories around an airport\\
    \bottomrule
    \end{tabular}
  \label{table:datasets}
\end{table}

\section{Experiments}

\subsection{Experimental setup}

\begin{table}[]
  \centering
  \caption{Effect of dimension adjustment}
    \begin{tabular}{@{}lrr@{}}
    \toprule
    $[\gamma_x, \gamma_y, \gamma_z]$ & RMSE & MSLL  \\
    \midrule
     $[0.1,0.1,0.1]$ & 1.368 & $-$1445  \\
     $[100,0.1,0.1]$ & {\bf 0.778} & $-$1426   \\
     $[100,100,100]$ & 1.496 & $-$1413   \\
    \bottomrule
    \end{tabular}
  \label{table:ard}
\end{table}

\begin{figure}[]
     \begin{subfigure}[b]{0.32\linewidth}
         \centering
         \includegraphics[width=1.\linewidth]{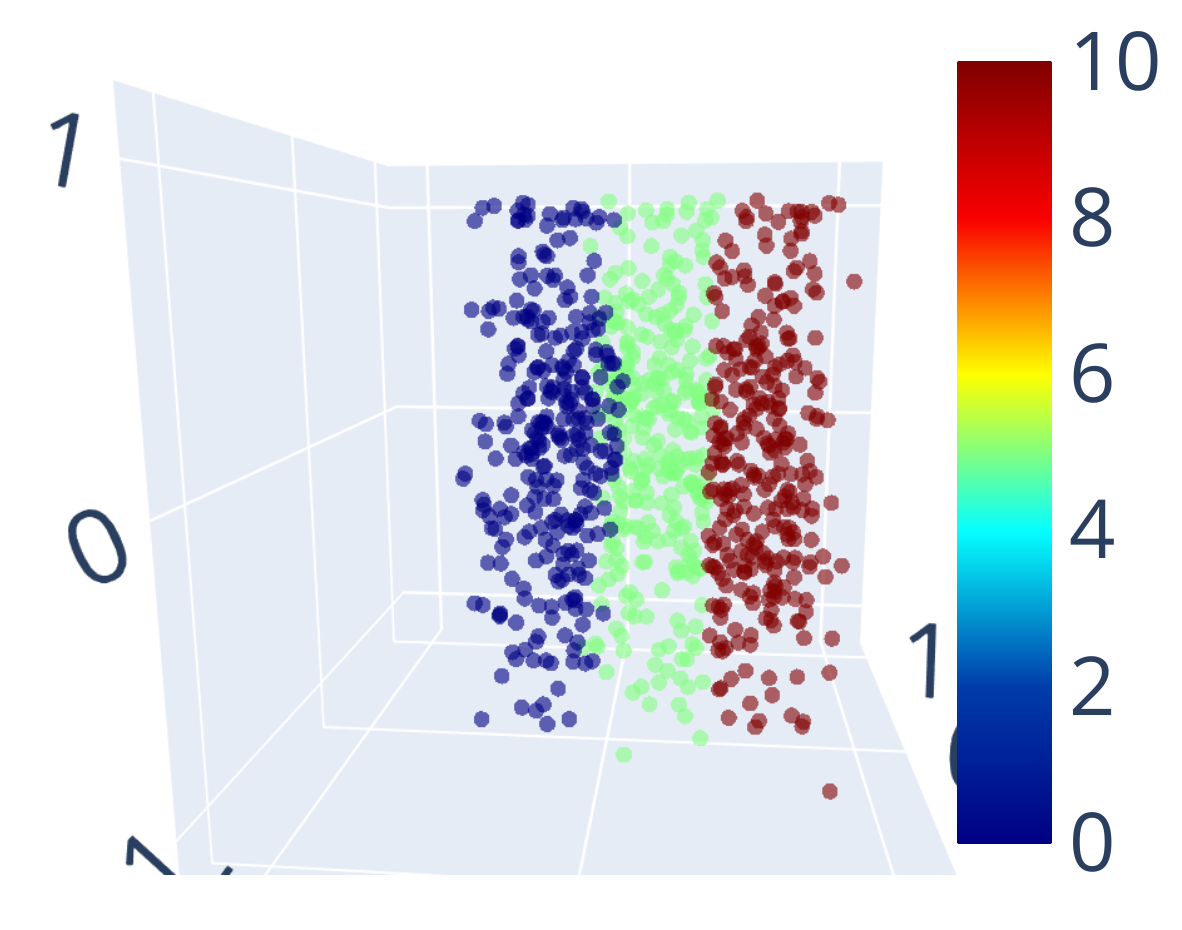}
         \caption{Chunk dataset}\label{fig:3a}	
    \end{subfigure}
    \begin{subfigure}[b]{0.32\linewidth}
         \centering
         \includegraphics[width=1.\linewidth]{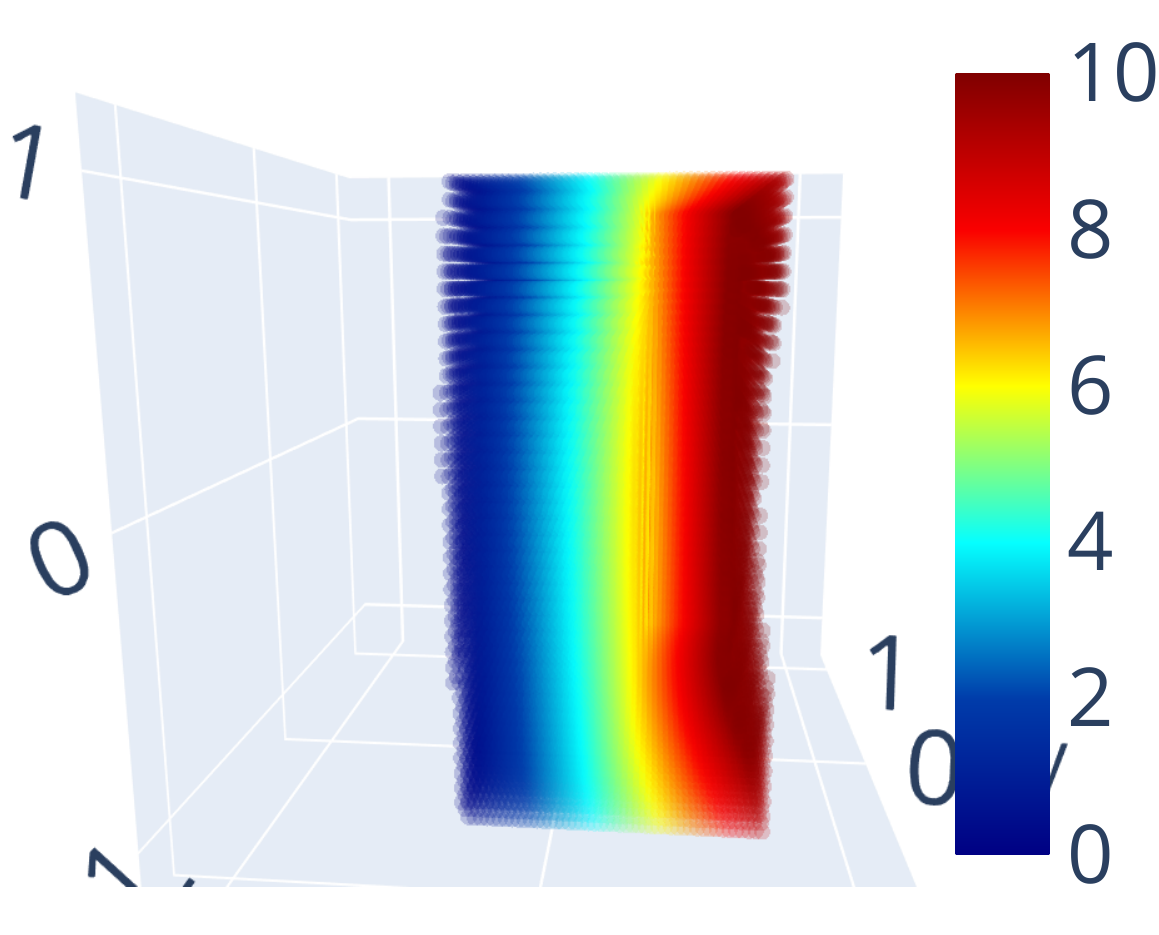}
         \caption{$[0.1,0.1,0.1]$}\label{fig:3b}
    \end{subfigure}
    \begin{subfigure}[b]{0.32\linewidth}
         \centering
         \includegraphics[width=1.\linewidth]{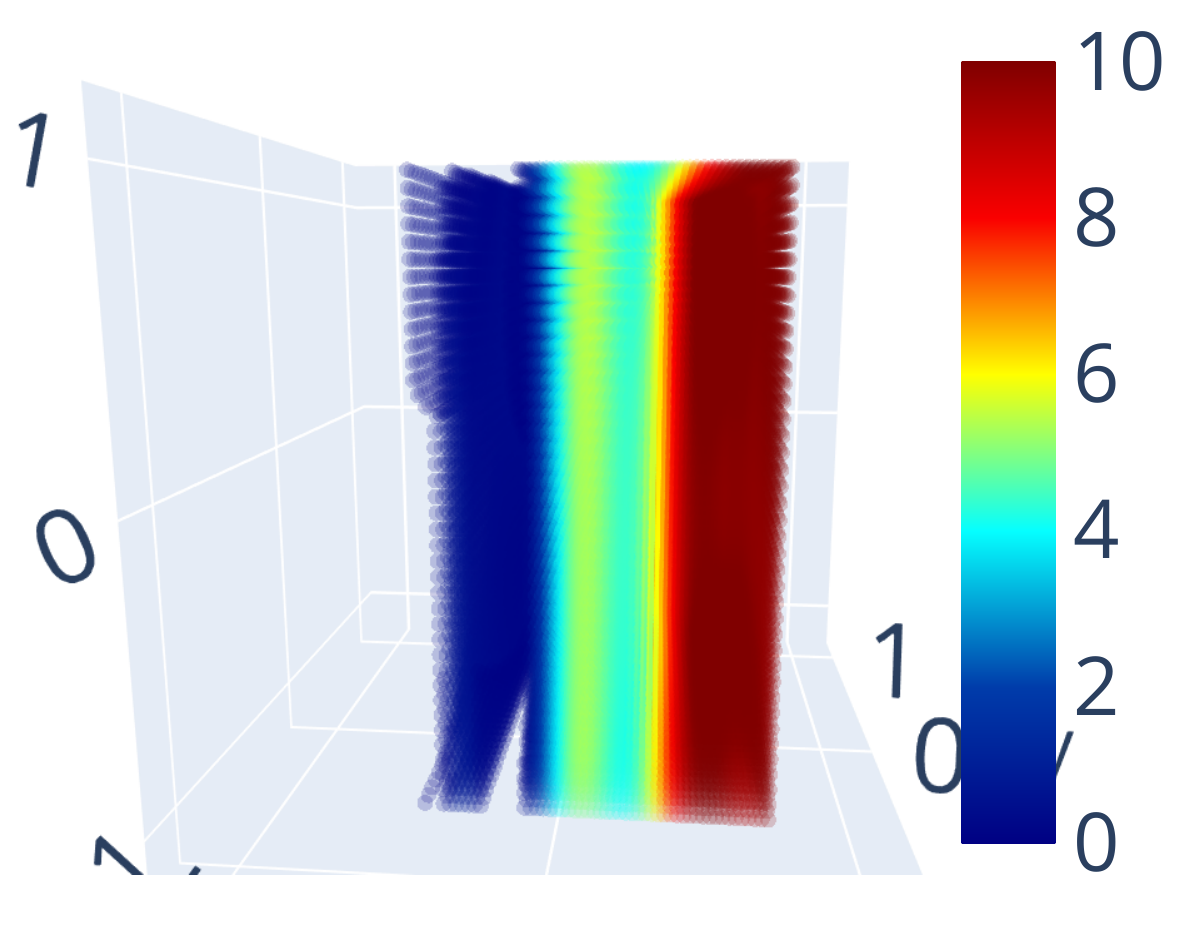}
         \caption{$[100,0.1,0.1]$}\label{fig:3c}
    \end{subfigure}
 \caption{Effect of dimension-adjusted kernels. Colors indicate velocity and the three values correspond to $[\gamma_x, \gamma_y, \gamma_z]$. In (c), when $\gamma_x > \gamma_y$ and $\gamma_x > \gamma_z$, the two boundaries between the three velocity clusters are much crispier.}
 \label{fig:ard}
\end{figure}

\begin{figure*}[]
\centering
     \begin{subfigure}[b]{0.32\linewidth}
         \centering
         \includegraphics[width=1.\linewidth]{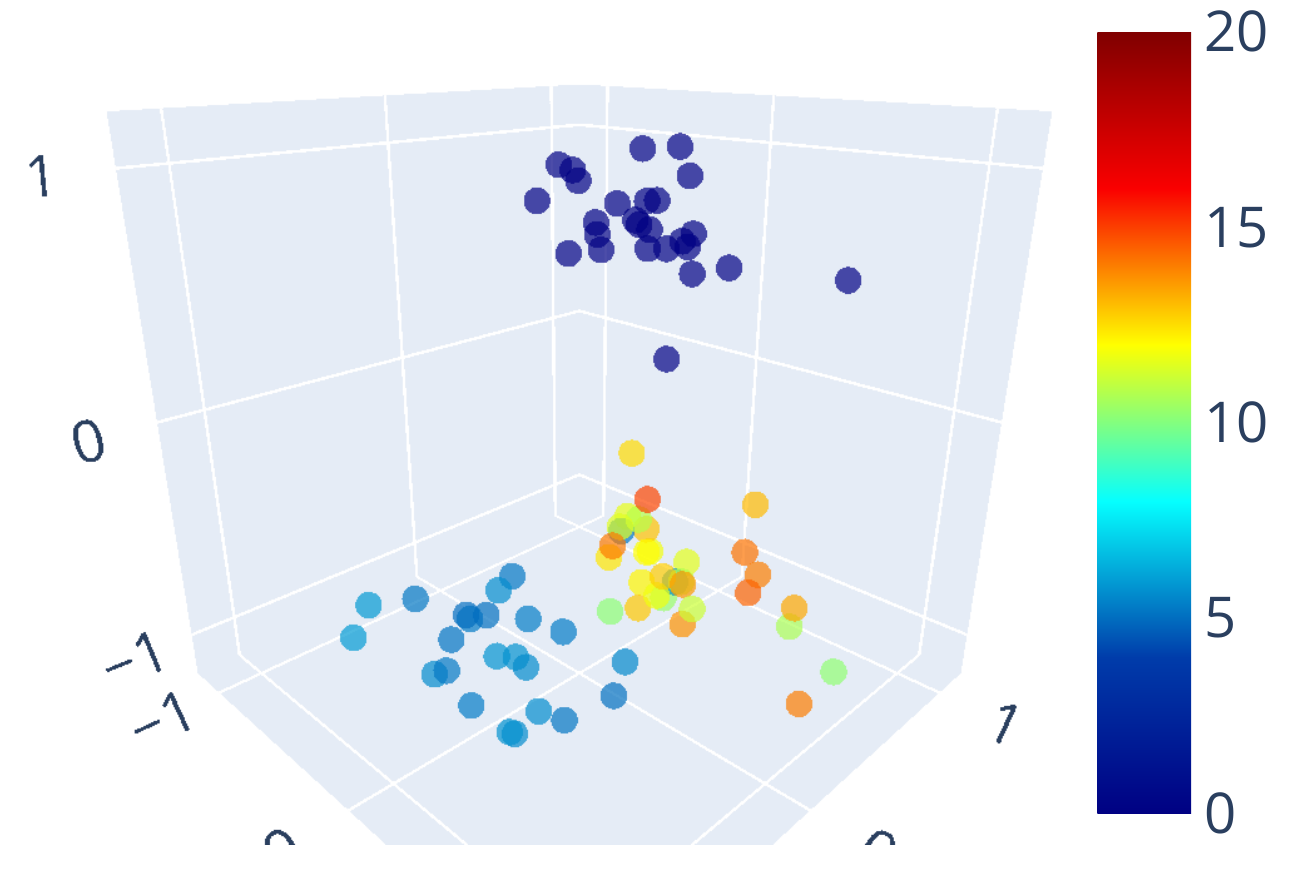}
         \caption{Blobs datasets}\label{fig:4a}
    \end{subfigure}
    \begin{subfigure}[b]{0.32\linewidth}
         \centering
         \includegraphics[width=1.\linewidth]{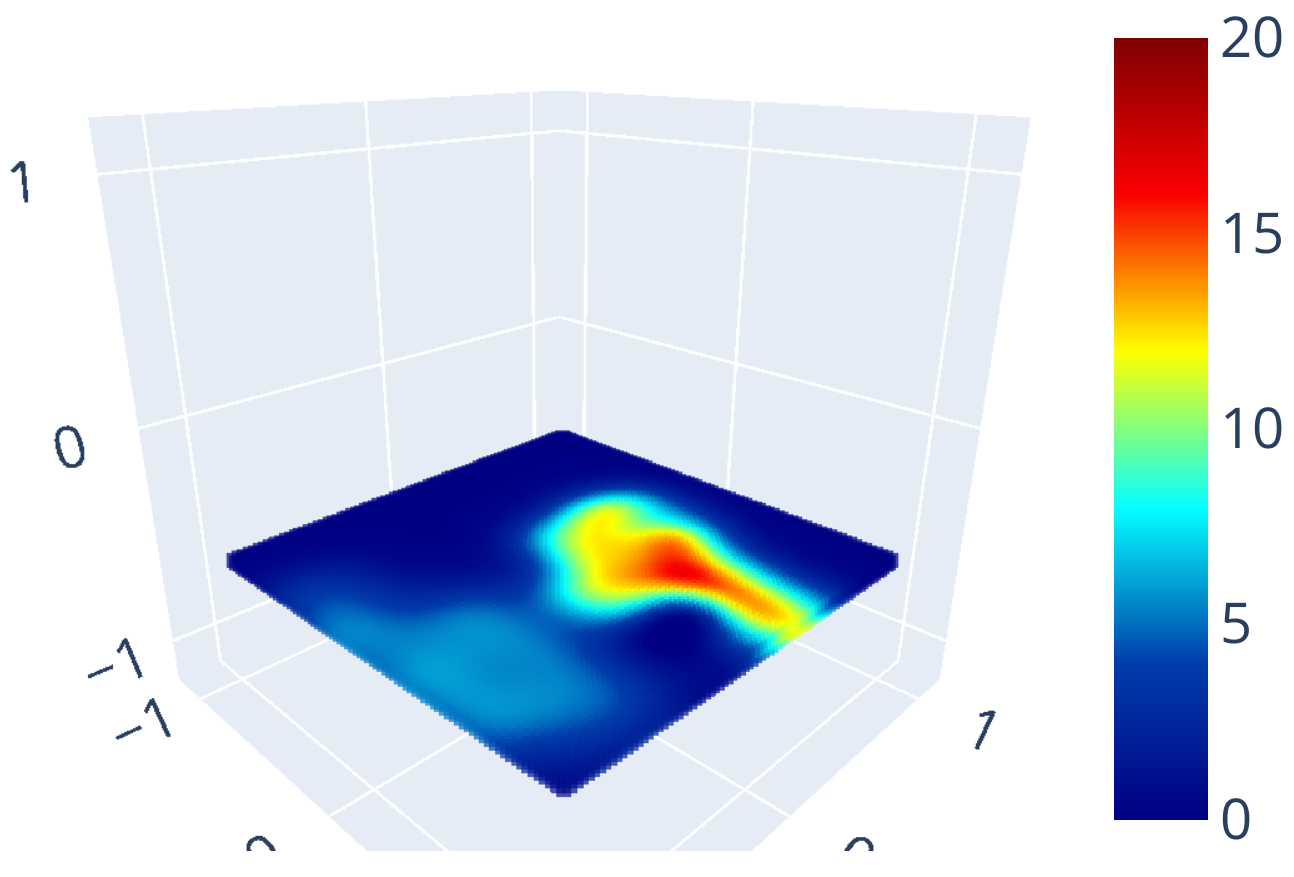}
         \caption{Predictive mean at $z=-0.6$}\label{fig:4b}
    \end{subfigure}
    \begin{subfigure}[b]{0.32\linewidth}
         \centering
         \includegraphics[width=1.\linewidth]{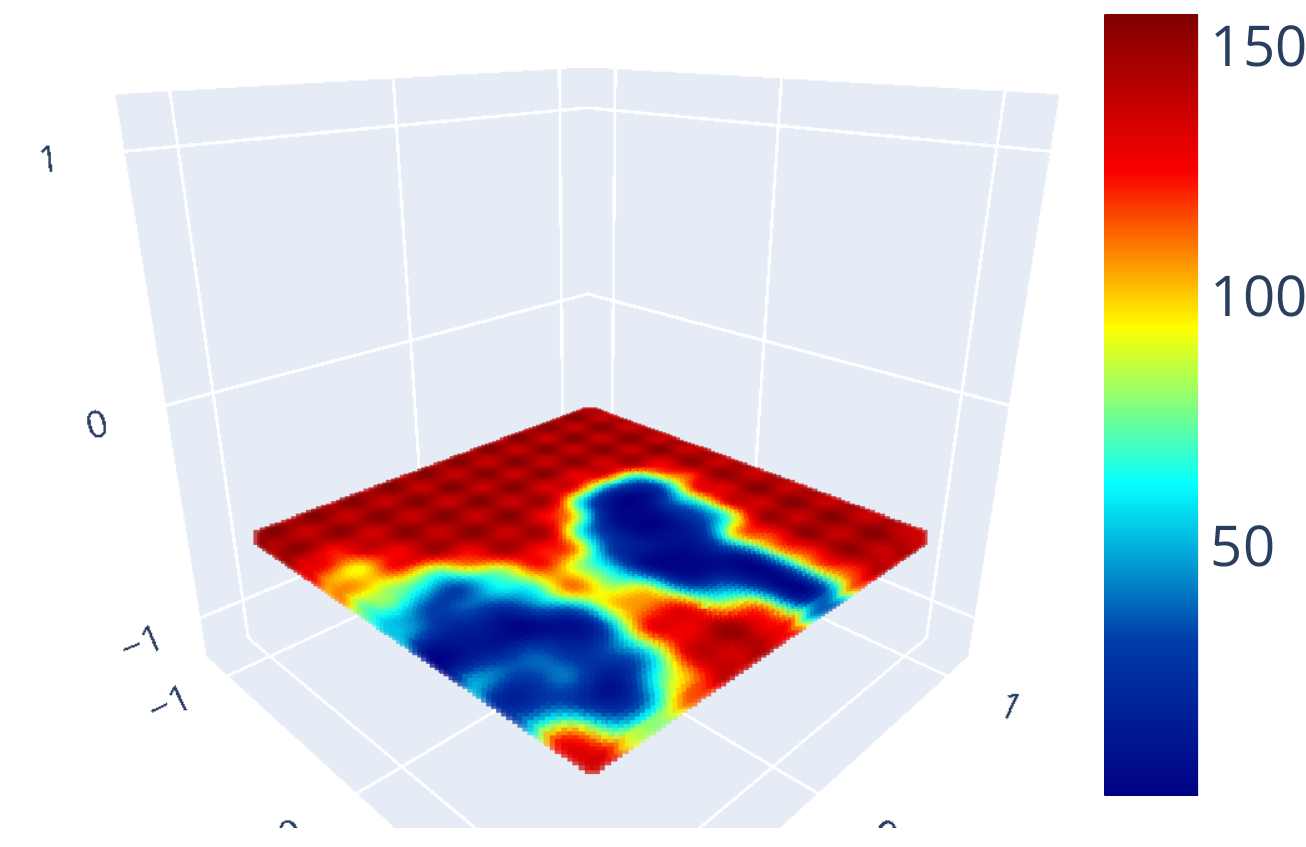}
         \caption{Predictive variance at $z=-0.6$}\label{fig:4c}
    \end{subfigure}
    \begin{subfigure}[b]{0.32\linewidth}
         \centering
         \includegraphics[width=1.\linewidth]{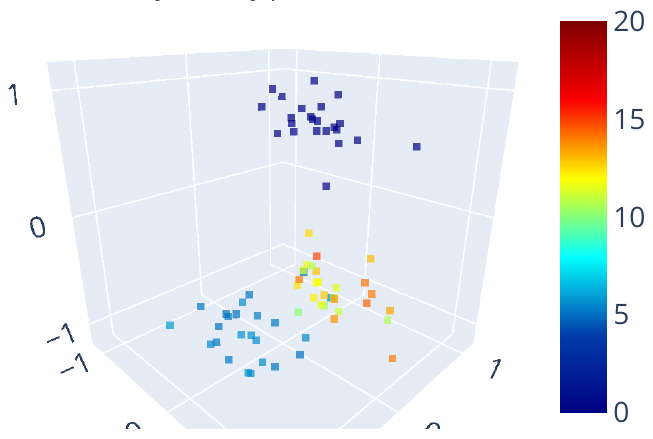}
         \caption{Prediction on training}\label{fig:4d}
    \end{subfigure}
    \begin{subfigure}[b]{0.32\linewidth}
         \centering
         \includegraphics[width=1.\linewidth]{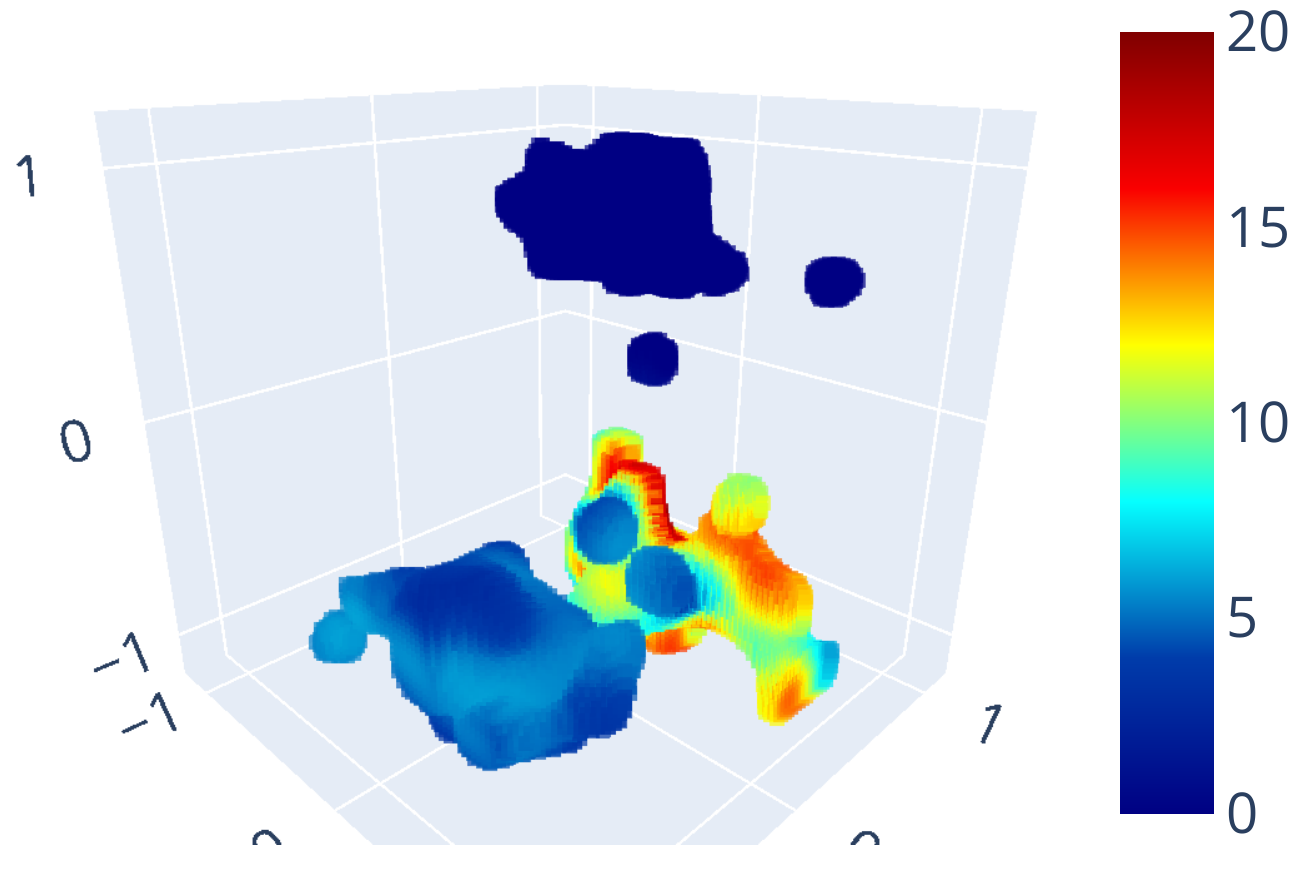}
         \caption{Filtered predictive mean}\label{fig:4e}
    \end{subfigure}
    \begin{subfigure}[b]{0.32\linewidth}
         \centering
         \includegraphics[width=1.\linewidth]{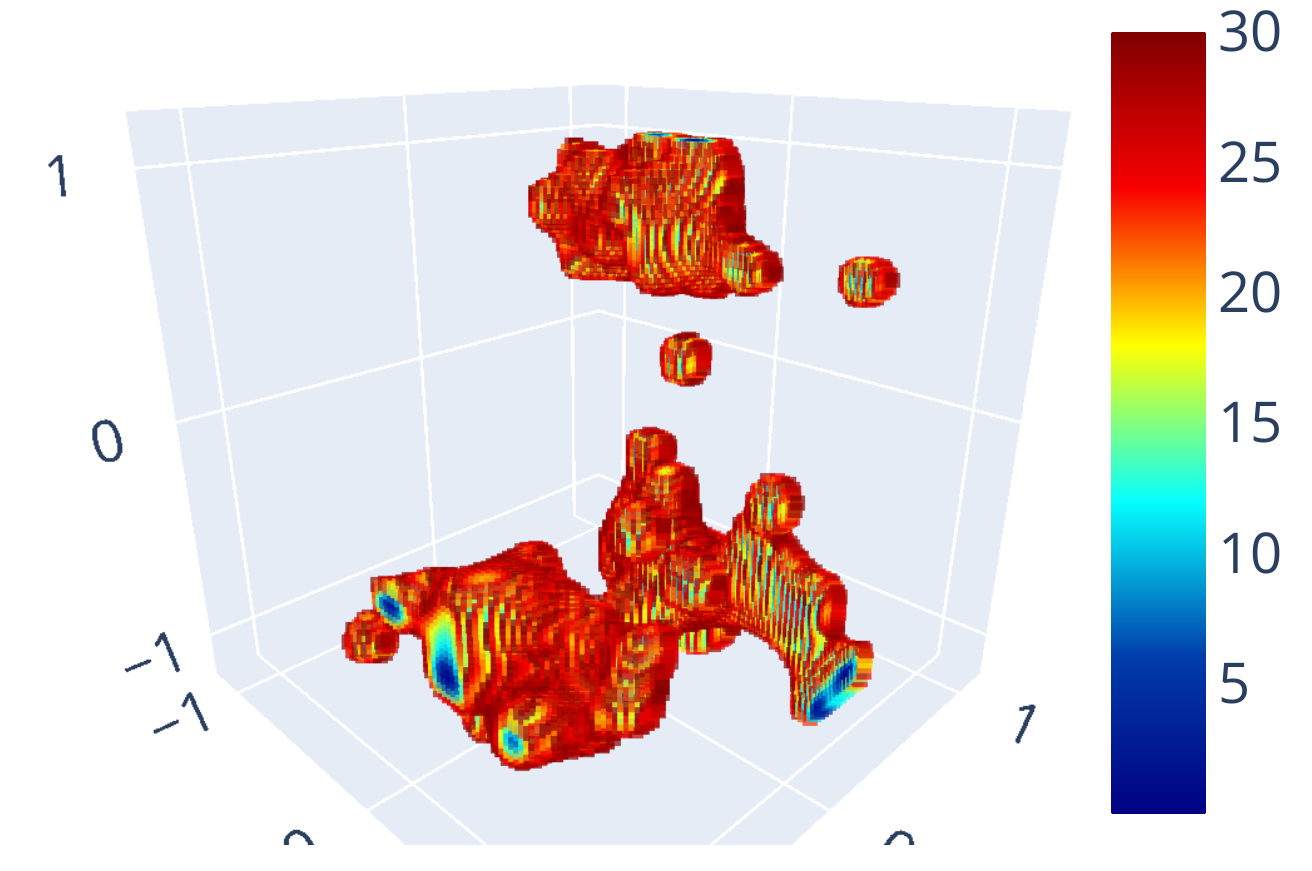}
         \caption{Filtered predictive variance}\label{fig:4f}
    \end{subfigure}
 \caption{Modeling the ${x}$ velocity component of the Blobs training dataset. (a) Each point in the 3D space has its own ${x}$ velocity. We want to predict the mean and variance of any other point using these data points. (b) Predicted mean velocity for $z=-0.6$. (c) Predicted variance of velocity for $z=-0.6$. Note that in areas where we do not have training data, the variance is high. (d) Predictions on the training dataset. (e)--(f) Velocity is predicted for the entire cubes but only high confidence predictions ($\sigma_* \le 30$) are shown.}
 \label{fig:6plots}
\end{figure*}

We studied the effectiveness of BDFs in constructing macroscopic and microscopic models from a variety of datasets summarized in Table~\ref{table:datasets}. These datasets were obtained from hi-fidelity simulators and real-world benchmark datasets. Since automotive radar is becoming increasingly popular in driverless cars, we included some of those datasets as well. The models we build using  small automotive radar datasets, Carla, Astyx, and nuScenes are microscopic because they only model their surroundings. These datasets have only a few data points per scan (Figure~\ref{fig:1c} and~\ref{fig:2b}). AirSim is a dataset that we generated using the AirSim simulator \cite{airsim2017fsr}. It contains 66859 data points of drone trajectories (Figure~\ref{fig:1b}) in a large area 1000$\times$400$\times$60 m$^3$. The airport dataset contains 128349 data points of real aircraft tracks within 30 nautical miles of the John F. Kennedy airport~\cite{Jung2019}.

Twenty percent of each dataset is used as the test dataset. For large areas, we normalized data to be in a cube of between $-1$ and $1$ and picked the hyperparameters $\gamma$ and grid distance using cross validation. The parameters $\alpha$ and $\beta$ were set to $10^{-2}$ and $10^2$, respectively. The code and video can be found at \url{https://github.com/RansML/BDF}. Experiments were run on a 3.30 GHz CPU. Since Gaussian process-based models are a common choice for modeling epistemic uncertainty in many robotics tasks~\cite{deisenroth2013gaussian,o2012gaussian}, we base-lined against full Gaussian process (FGP)~\cite{rasmussen2003gaussian} and its scalable approximations such as subset of data Gaussian process (SGP)~\cite{herbrich2003fast} and more recent big data Gaussian process (BGP)~\cite{senanayake2017learning}. GPflow~\cite{GPflow2017} was used for benchmarking. 

\subsection{Effect of dimension adjustment}
As the first experiment, we verify that dimension-adjusted kernels are useful to maintain sharp velocity transitions. For this purpose, we use the the Chunks dataset as it has sharp velocity transitions. Figure~\ref{fig:ard} shows that we get much crisper edges when we use a higher $\gamma$ in the $x$ direction. This is because gamma controls the contribution from each direction. Table~\ref{table:ard} further corroborates that higher $\gamma$ for the $x$-axis has the least root mean squared error (RMSE) for a similar mean standardized log loss (MSLL)~\cite{rasmussen2003gaussian}.

\subsection{Runtime and accuracy}

\begin{table}[]
  \centering
  \caption{Runtime and accuracy}
    \begin{tabular}{@{}ccrcc@{}}
    \toprule
       Dataset & Method & Train time (s) & Query time (s) & RMSE \\
      \midrule
      \parbox[t]{8mm}{\multirow{3}{*}{Chunks}} 
      & BDF & {\bf 6.805} & 0.673  & 0.778 \\
      & BGP & 1819.469 & 0.075 & 1.258  \\
      & SGP & 9.169 & 0.163 & 1.252  \\
      & FGP & 9.169 & 0.163 & 1.252 \\
      \midrule
      \parbox[t]{8mm}{\multirow{3}{*}{Blobs}} 
      & BDF & {\bf 0.448} & 0.037 & 1.119  \\
      & BGP & 89.411 & 0.038 & 0.669  \\
      & SGP & 1.839 & 0.036 & 0.688 \\
      & FGP & 1.839 & 0.036 & 0.688 \\
      \midrule
        \parbox[t]{8mm}{\multirow{3}{*}{Carla}} 
      & BDF & {\bf 0.726} & 0.044 & 4.581   \\
      & BGP & 66.618 & 0.051 & 5.744   \\
      & SGP & 1.849 & 0.032 & 5.992  \\
      & FGP & 1.849 & 0.032 & 5.992  \\
      \midrule
        \parbox[t]{8mm}{\multirow{3}{*}{Astyx}} 
      & BDF & 5.517 & 0.087 & 0.329   \\
      & BGP & 2575.300 & 0.059 & 0.275   \\
      & SGP & {\bf 4.983} & 0.110 & 0.275  \\
      & FGP & {\bf 4.983} & 0.110 & 0.275  \\
      \midrule
        \parbox[t]{8mm}{\multirow{3}{*}{nuScenes}} 
      & BDF & {\bf 0.001} & 0.010 & 0.370   \\
      & BGP & 45.651 & 0.019 & 0.395   \\
      & SGP & 0.587 & 0.016 & 0.396 \\
      & FGP & 0.587 & 0.016 & 0.396 \\
      \midrule
        \parbox[t]{8mm}{\multirow{3}{*}{AirSim}} 
      & BDF & {\bf 16.213} & 0.533 & 0.514 \\
      & BGP & 2979.046 & 0.125 & 0.736   \\
      & SGP & 163.545 & 1.455 & 0.154  \\
      & FGP & $\infty$ (est.) & $\infty$ (est.) & n/a  \\
      \midrule
        \parbox[t]{8mm}{\multirow{3}{*}{Airport}} 
      & BDF & {\bf 16.129} & 0.856 & 1.801   \\
      & BGP & 2966.236 & 0.154 & 2.139  \\
      & SGP & 214.030 & 4.232 & 1.670  \\
      & FGP & $\infty$ (est.) & $\infty$ (est.) & n/a  \\
      \bottomrule
      \multicolumn{5}{l}{\footnotesize{Due to limited scalability of benchmarks, only 3.8\% and 4.3\% of}}\\
      \multicolumn{5}{l}{\footnotesize{data in AirSim and JFK, respectively, was used for all four methods.}}\\
    \end{tabular}
  \label{table:metrics}
\end{table}

\begin{figure}[]
\centering
     \begin{subfigure}[b]{0.7\linewidth}
         \centering
         \includegraphics[width=1.\linewidth]{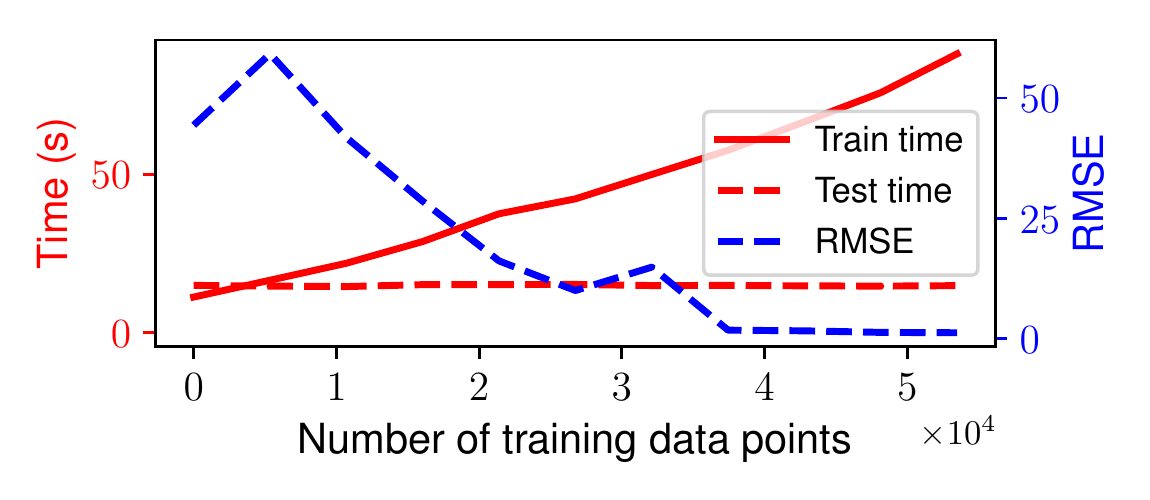}
         \caption{AirSim dataset}\label{fig:5a}	
    \end{subfigure}
    \begin{subfigure}[b]{0.7\linewidth}
         \centering
         \includegraphics[width=1.\linewidth]{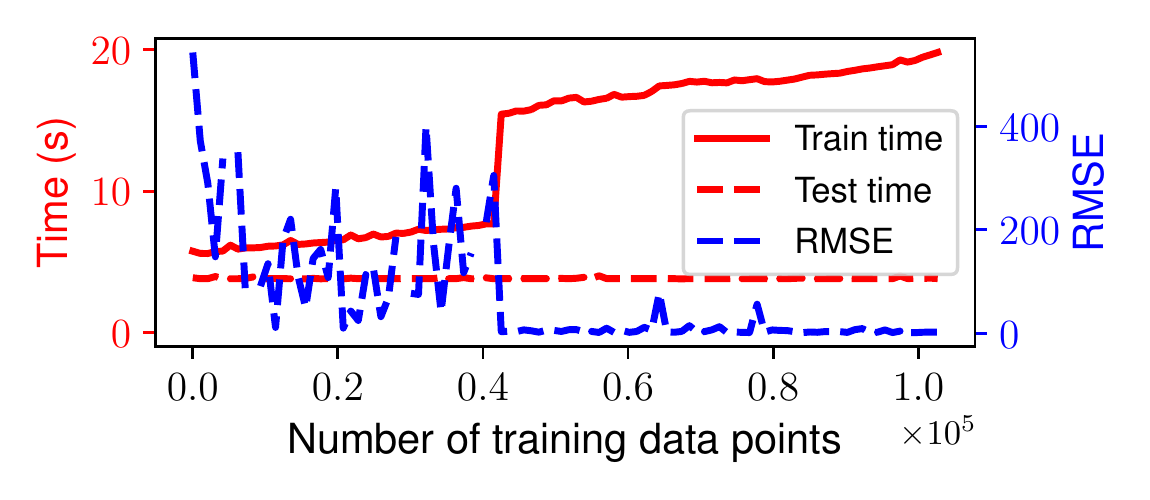}
         \caption{Airport dataset}\label{fig:5b}
    \end{subfigure}
 \caption{Effect of the increasing size of the dataset for the macroscopic model.}
 \label{fig:seq}
\end{figure}

\begin{figure}[]
\centering
     \begin{subfigure}[b]{0.7\linewidth}
         \centering
         \includegraphics[width=1.\linewidth]{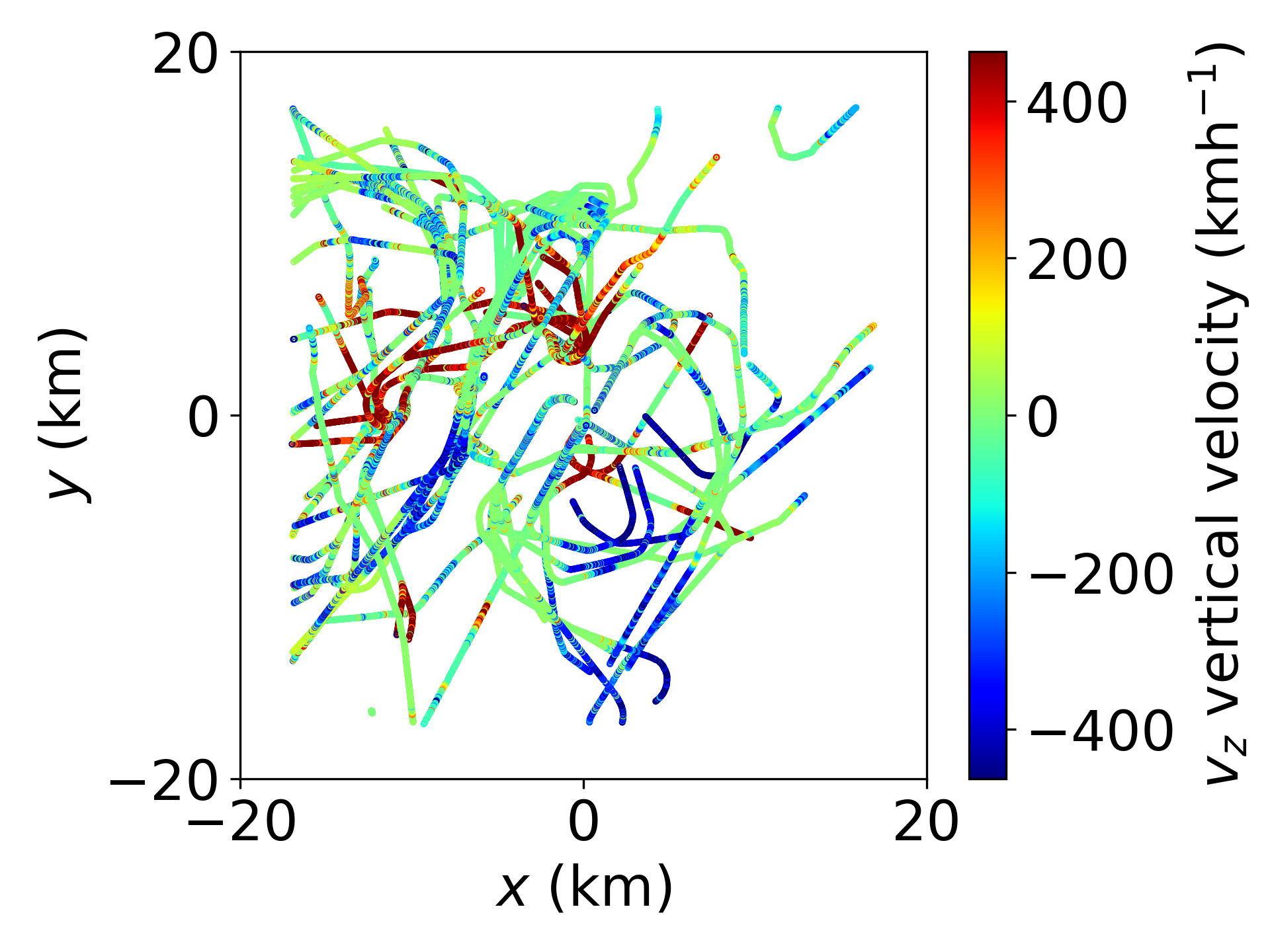}
         \caption{Ground truth}\label{fig:6a}	
    \end{subfigure}
    \begin{subfigure}[b]{0.7\linewidth}
         \centering
         \includegraphics[width=.95\linewidth]{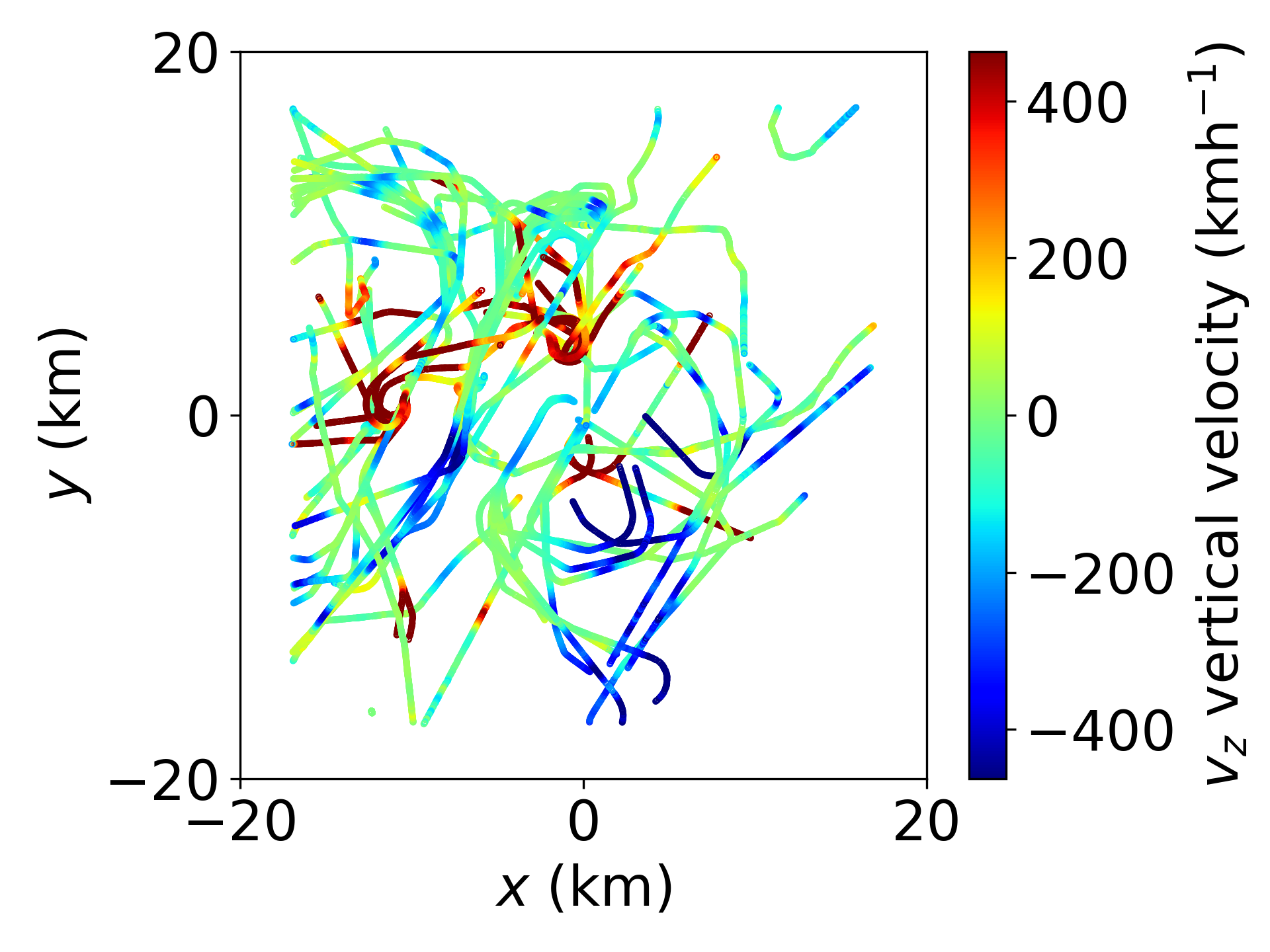}
         \caption{Mean prediction}\label{fig:6b}
    \end{subfigure}
 \caption{Bird's-eye view of 100 trajectories of the Airport dataset. The mean predictions of vertical velocities are similar to the ground truth.}
 \label{fig:Airport}
\end{figure}

\begin{figure}[]
\centering
     \begin{subfigure}[]{0.7\linewidth}
         \centering
         \includegraphics[width=1.\linewidth]{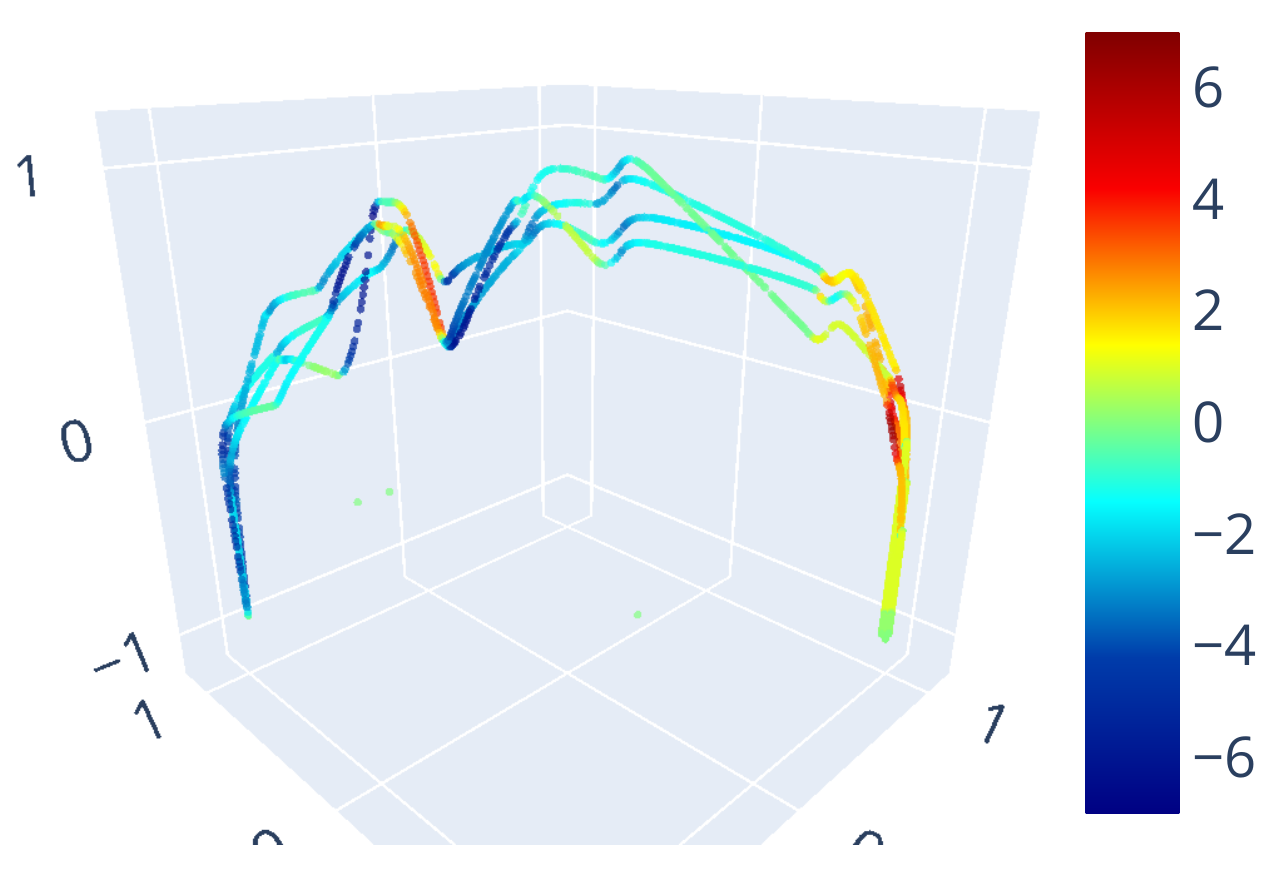}
         \caption{Ground truth trajectories}\label{fig:7a}	
    \end{subfigure}
    \begin{subfigure}[]{0.7\linewidth}
         \centering
         \includegraphics[width=1.\linewidth]{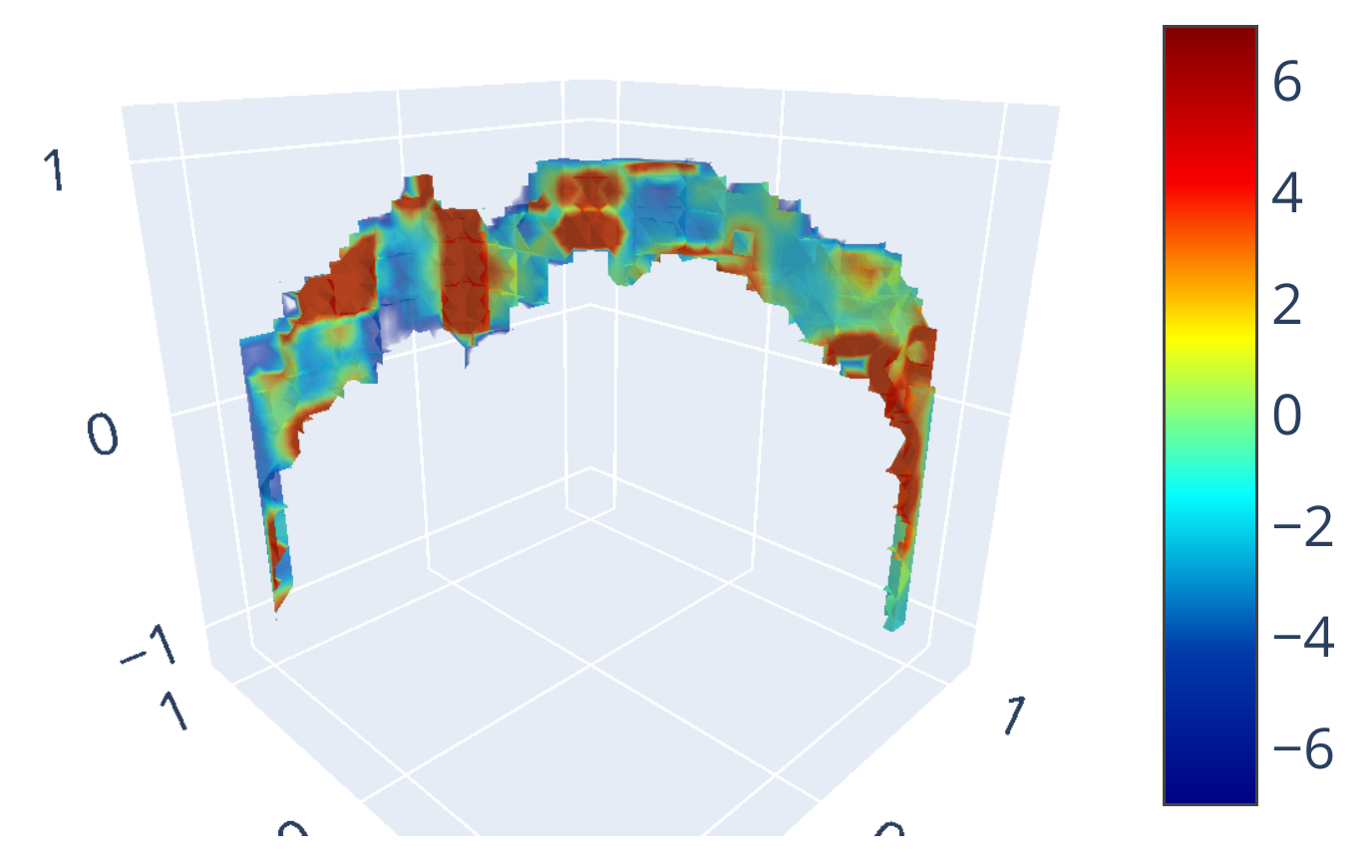}
         \caption{Safe velocity tube}\label{fig:7b}
    \end{subfigure}
 \caption{One of the ``air roads'' simulated in AirSim. Color indicates vertical velocity. A ``3D tube'' of velocity is obtained by filtering mean predictions below a given variance threshold. These tubes can be used for risk-aware control~\cite{cannon2012stochastic}.}
 \label{fig:tube}
\end{figure}

Figure~\ref{fig:6plots} shows a detailed example of how we can model the 3D space. Observe that the variance in areas where we do not have data is higher, indicating the epistemic uncertainty. Metrics for each dataset are reported in Table~\ref{table:metrics}. For almost all datasets, our model has the least training time to achieve similar accuracy to other methods. For small datasets (points < 1000), SGP is equivalent to FGP as the subset is the same as the full dataset. Since BGP is based on a stochastic gradient approach, it is typically harder to optimize compared to other analytical forms and is not extremely useful in small data settings.

The efficiency of our model is more pronounced in large datasets such as Airport (Figure~\ref{fig:Airport}) and AirSim (Figure~\ref{fig:tube}). BDF is at least 10 times faster than the second best performing model. This is because its asymptotic computational complexity is $\mathcal{O}(M^3)$ for $M$ kernels. That is, the speed depends only on the number of kernels but not on the number of data points. In contrast, FGPs have a $\mathcal{O}(N^3)$ memory complexity for both training and query for $N$ data points. In their sparse approximations, $P$ \emph{inducing points} are used to represent the key points in the dataset~\cite{quinonero2005unifying}. For $P \ll N$, the computational complexities of BGP and SGP are $\mathcal{O}(P^2N)$ and $\mathcal{O}(P^3)$, respectively. Although SGP, in theory, has a similar asymptotic computational complexity, the major drawback of SGP is that it discards a large amount of data to achieve this speedup. In our method, every data point has an equal contribution when training the model. 

Although BDFs have $\mathcal{O}(M^3)$ complexity, in practice, when $M$ is finite (around 1000 in most of our experiments), we observe a slight increase with the number of data points due to the matrix product in (\ref{eq:post_mu}). This can be observed in Figure~\ref{fig:seq} in which we separately train the model for an increasing number of data points. This slight increase of time is negligible compared to BGP and FGP in which the training time is 50 minutes before failing.

Since BDFs are parametric models that use kernels, we can update parts of the model or combine various models. This is because kernels are similarity functions, and therefore the weight parameters of a kernel located at $\tilde{\mathbf{x}}$ are not affected by data far away from them. Similarly, a large-scale model can be easily decomposed to create light-weight models that cater to only a designated area of the environment.

\section{Conclusions}
This paper presented Bayesian Dynamic Fields to model velocity in the 3D space. The velocity of the environment is represented as a continuous function that can be queried at arbitrary resolutions. The training procedure is equally suitable for both small and big data regimes, making it suitable to build microscopic and macroscopic transportation models. The model captures both velocity estimates as well as the uncertainty associated with those estimates. In conjunction with common environment representations such as occupancy maps in robotics, in the future, we will use the uncertainty estimates of velocity for decision-making under uncertainty and safety analysis~~\cite{cannon2012stochastic}.

\section*{Acknowledgment}
The authors thank Soyeon Jung for assisting with preprocessing the aviation radar dataset. Toyota Research Institute (TRI) provided funds to assist the authors with their research, but this article solely reflects the opinions and conclusions of its authors and not TRI or any other Toyota entity.

\renewcommand*{\bibfont}{\footnotesize}
\printbibliography


\end{document}